\documentclass{article}
\usepackage{graphicx} 
\usepackage{amsmath,amssymb,amsthm}
\usepackage{fullpage}
\usepackage[hcentering]{geometry}
\usepackage{caption} \usepackage{subcaption}
\usepackage{graphicx}
\usepackage{float}
\usepackage{placeins}
\usepackage{paralist}
\usepackage[utf8]{inputenc}   
\usepackage[T1]{fontenc} 
\newif\iffinal
\finaltrue
\newif\ifarxiv
\arxivfalse

\usepackage{paralist}
\usepackage{url}

\usepackage{aliascnt}
\def\NewTheorem#1#2{%
  \newaliascnt{#1}{theorem}
  \newtheorem{#1}[#1]{#2}
  \aliascntresetthe{#1}
  \expandafter\def\csname #1autorefname\endcsname{#2}
}
 \newtheorem{theorem}{Theorem}[section]
\NewTheorem{lemma}{Lemma}
\NewTheorem{corollary}{Corollary}
\NewTheorem{conjecture}{Conjecture}
\NewTheorem{observation}{Observation}
\NewTheorem{definition}{Definition}
\NewTheorem{assumption}{Assumption}
\NewTheorem{proposition}{Proposition}
\NewTheorem{remark}{Remark}
\NewTheorem{claim}{Claim}
\NewTheorem{fact}{Fact}
\NewTheorem{example}{Example}

\newcommand{\hide}[1]{}

\DeclareMathOperator{\lmax}{\ell_{max}}

\usepackage{color}
\usepackage[usenames,dvipsname s,svgnames,table]{xcolor}

\definecolor{darkred}{rgb}{0.5,0,0}
\definecolor{lightblue}{rgb}{0,0.4,0.8}
\definecolor{darkgreen}{rgb}{0,0.5,0}

\usepackage[colorlinks=true, linkcolor=blue, urlcolor=blue, citecolor=darkgreen]{hyperref}

\usepackage[utf8]{inputenc}

\begin{document}

\title{Abstraction in Neural Networks}

\author{Nancy Lynch \\
	Massachusetts Institute of Technology \\
	Department of Electrical Engineering and Computer Science \\
    32 Vassar St, Cambridge, MA 02139 \\ 
    \texttt{lynch@csail.mit.edu}\footnote{This work was supported by the National Science Foundation under grants CCR-2139936 and CCR-2003830.} \\	
}

\date{August 4, 2024}
\maketitle

\begin{abstract}
We show how brain networks, modeled as \emph{Spiking Neural Networks}, can be viewed at different levels of abstraction. 
Lower levels include complications such as failures of neurons and edges.
Higher levels are more abstract, making simplifying assumptions to avoid these complications.
We show precise relationships between the executions of networks at the different levels, which enables us to understand the behavior of the lower-level networks in terms of the behavior of the higher-level networks.

We express our results using two abstract networks, $\mathcal{A}_1$ and $\mathcal{A}_2$, one to express firing guarantees and the other to express non-firing guarantees, and one detailed network $\mathcal{D}$.  
The abstract networks contain reliable neurons and edges, whereas the detailed network has neurons and edges that may fail, subject to some constraints.
Here we consider just initial stopping failures.

To define these networks, we begin with abstract network $\mathcal{A}_1$ and modify it systematically to obtain the other two networks.
To obtain $\mathcal{A}_2$, we simply lower the firing thresholds of the neurons. 
To obtain $\mathcal{D}$, we introduce failures of neurons and edges, and incorporate redundancy in the neurons and edges in order to compensate for the failures.
We also define corresponding inputs for the networks, and corresponding executions of the networks.

We prove two main theorems, one relating corresponding executions of $\mathcal{A}_1$ and $\mathcal{D}$ and the other relating corresponding executions of $\mathcal{A}_2$ and $\mathcal{D}$.
Together, these give both firing and non-firing guarantees for the detailed network $\mathcal{D}$.  
We also give a third theorem, relating the effects of $\mathcal{D}$ on an external reliable actuator neuron to the effects of the abstract networks on the same actuator neuron. 
\end{abstract}

\section{Introduction}


This work is part of an ongoing effort to understand the behavior of brain networks in terms of mathematical distributed algorithms, which we call \emph{brain algorithms}.  The overall aim is to describe mechanisms that are realistic enough to explain actual brain behavior, while at the same time simple and clear enough to enable simulations and mathematical analysis.
Our prior work on this project includes (among other things) study of brain decision-making using \emph{Winner-Take-All} mechanisms~\cite{LMP19, SuCL19}, data compression using random projection~\cite{DBLP:conf/innovations/HitronLMP20}, learning and recognition for hierarchically-structured concepts~\cite{LM21, DBLP:conf/sirocco/LynchM23}, and the interaction between symbolic and intuitive computation~\cite{lynch2022symbolicknowledgestructuresintuitive}.
We have also worked on defining mathematical underpinnings of the area, in terms of simple \emph{Spiking Neural Network (SNN)} models~\cite{Lynch2022}.


An obstacle to work in the area is that real brain networks have complex structure and behavior, including such complications as neuron and synapse failures.
Theoretical models that capture such complexity are similarly complicated.
For example, work by Valiant~\cite{Valiant}, work on the \emph{assembly calculus}~\cite{PapadimitriouVempala, PVMM20}, and some of our earlier work~\cite{lynch2024multineuron} view the world in terms of individual "concepts", and represent each concept with many neurons, using redundancy to mitigate the effects of various anomalies.
However, designing, understanding, and analyzing algorithms in terms of such complex models can be difficult.

To make algorithm design and analysis easier, much of the theoretical work on brain algorithms, including ours, has assumed simplified models, in which network elements are reliable.
But that means that there is a gap between the simplified theoretical models and the more complicated actual brain networks.


This paper illustrates one approach to bridging that gap, namely, describing brain algorithms in two ways, at different \emph{levels of abstraction}.  Lower levels can include complications such as failures and noise.
Higher levels can be more abstract, making simplifying assumptions to avoid some of the complications.
The problem then is to reconcile the two levels, by relating the complex, detailed models to the abstract models.
This approach breaks down the study of brain algorithms into two parts:  studying the abstract models on their own, and mapping the detailed models to the abstract models.
Such an approach is common in other areas of computer science, such as the theory of distributed systems; the book~\cite{10.5555/2821576} contains some examples.
The approach is well developed in the formal theory of concurrency; versions of this approach are described (for example) in~\cite{DBLP:conf/podc/LynchT87, DBLP:series/synthesis/2010Kaynar, DBLP:journals/njc/SegalaL95}.

\paragraph{Background:}
The present paper has evolved from two previous papers~\cite{lynch2024multineuron, lynch2024abstraction}.
The first of these~\cite{lynch2024multineuron} describes representation of hierarchically structured concepts in a layered neural network, in such a way as to tolerate some random neuron failures and random connectivity between consecutive layers and within layers.
The representations use multiple neurons for each concept.
The representations are designed to support efficient recognition of the concepts, and to be easily learned using Hebbian-style learning rules.
The analysis in~\cite{lynch2024multineuron} is fairly involved, requiring considerable probabilistic reasoning, mainly Chernoff and union bounds.

The paper~\cite{lynch2024abstraction} takes another look at most of the results of~\cite{lynch2024multineuron}, reformulating them in terms of levels of abstraction.
It defines abstract networks that represent each concept with a single reliable neuron, and include full connectivity between representing neurons in consecutive layers.
Detailed networks include neuron failures and reduced connectivity.
Correctness requirements are formulated for both abstract and detailed networks.  Correctness of abstract networks is argued directly, whereas correctness of detailed networks is shown by mapping the detailed networks to corresponding abstract networks, using formal mappings.
One simplification here is that we replace the randomness of failures and connectivity from~\cite{lynch2024multineuron} with constraints that could be shown to hold with high probability, based on Chernoff arguments.
This allows us to avoid explicit probabilistic reasoning.


\paragraph{What we do here:}
In this paper, we extend the results of~\cite{lynch2024abstraction}.
Now instead of considering just layered networks that represent hierarchical concepts, we consider arbitrary neural network digraphs.
In~\cite{lynch2024abstraction}, we compared abstract and detailed networks, for the special case of hierarchical concept graphs.  But then we observed that the ideas are more general, and could apply to a larger class of abstract networks.
This is what we explore here.

As it turns out, we make some changes from the approach in~\cite{lynch2024abstraction}, in order to simplify the presentation.
In particular, we unify the treatment of neuron failures and connectivity by considering two kinds of failures, neuron failures and edge failures.
This makes sense since now we are considering network digraphs with arbitrary connectivity.
Also, we do not try to extract general, formal notions of mappings between arbitrary networks at different levels of abstraction, but just present our mapping claims within the statements of the mapping theorems.


However, most of the treatment is analogous.
We again use two abstract networks, $\mathcal{A}_1$ and $\mathcal{A}_2$, one to express firing guarantees and the other to express non-firing guarantees, and a detailed network $\mathcal{D}$.  
The abstract networks contain reliable elements (neurons and edges), whereas $\mathcal{D}$ has neurons and edges that may fail, subject to some constraints.
As before, we consider just initial stopping failures.

To define these networks, we begin with the abstract network $\mathcal{A}_1$ and modify it systematically to obtain the other two networks.
To obtain $\mathcal{A}_2$, we simply lower the firing thresholds of the neurons. 
To obtain $\mathcal{D}$, we introduce failures of neurons and edges, with certain constraints, and incorporate redundancy in the neurons and edges in order to compensate for the failures.
We also define corresponding input conventions for the networks at different levels of abstraction, and use these to define corresponding executions of the different networks.

Then we prove two \emph{mapping theorems}, one relating corresponding executions of $\mathcal{A}_1$ and $\mathcal{D}$ and the other relating corresponding executions of $\mathcal{A}_2$ and $\mathcal{D}$.
Together, these give both firing and non-firing guarantees for the detailed network $\mathcal{D}$, based on corresponding guarantees for the abstract networks.  We give one more theorem, relating the effects of $\mathcal{D}$ on an external \emph{actuator neuron}, to the effects of the abstract networks on the same actuator neuron. 


The results in this paper are presented in terms of a particular Spiking Neural Network model, in which the state of a neuron simply indicates whether or not the neuron is currently firing, and a particular failure model, with initial stopping failures.
It should be possible to extend the results to more elaborate failures.
Moreover, similar results should be possible for other types of network models, such as those in which neurons have real-valued state components representing their activity levels.
%
The key idea in all cases is that the abstract models are reliable, with reliable neurons and edges.  The detailed models include failures of neurons and edges, and include multiple neurons to compensate for the failures.  The key results in all cases should be mappings that give close relationships between the behaviors of the detailed and abstract networks.


\paragraph{Roadmap:}
The rest of the paper is organized as follows.
In Section~\ref{sec: network}, we present our general network models, both abstract and detailed.
In Section~\ref{sec: correspondence}, we define correspondences between the abstract and detailed networks.
Section~\ref{sec: mappings} contains our mapping theorems.
Finally, Section~\ref{sec: conclusions} contains our conclusions.

\section{Network Model}
\label{sec: network}

We consider abstract and detailed versions of the network model. 
The only formal difference is that the detailed models include failures of neurons and edges, while in the abstract model, these components are reliable.

\subsection{Abstract network model}

An instance of the abstract model is a weighted directed graph $G = (V,E)$ with possible self-loops.
Each node of the digraph has an associated \emph{neuron}, which is modeled as a kind of state machine.
Some of the neurons are designated as \emph{input neurons}.
Input neurons do not have any incoming edges, in particular, they have no self-loops.

Each neuron $v$ has the following state component:
\begin{itemize}
    \item $firing$, with values in $\{0,1\}$.
\end{itemize} 
$firing = 1$ means that neuron $v$ is firing, and $firing = 0$ means that it is not firing.
Each non-input neuron also has the following associated information:
\begin{itemize}
    \item $threshold$, a real number.
\end{itemize}
For a neuron $v$, we write $firing^v$ and $threshold^v$.

Each edge has the following associated information:
\begin{itemize}
    \item $weight$, a real number.
\end{itemize}
For an edge $e = (u,v)$, we write $weight^e$ or $weight^{(u,v)}$.
The weights might be negative, which is how we model inhibition.

We will use notation $u,v,w,x,y,z$ for neurons.
When we want to make a distinction, we will use $u,v,w$ for neurons in abstract networks and $x,y,z$ for neurons in detailed networks.

Inputs are presented to the network at the input neurons. 
The input consists of setting a subset of the input neurons' $firing$ state compoonents to $1$, and the others to $0$, at each time $t$. 
We assume that the $firing$ state components of the input neurons are set by some external force.
We write $firing^v(t)$ for the value of the $firing$ state component of neuron $v$ at time $t$.


A \emph{configuration} of the network is a mapping from the neurons in the network to $\{0,1\}$.  This gives the $firing$ state component values of all of the neurons in the network.
The \emph{initial configuration} of the network is determined as follows.
For input neurons, the state (i.e., the value of the $firing$ component) is determined externally.  For non-input neurons, it is specified as part of the network definition.

Given a configuration $C$ of the network at time $t$, which we denote by $C(t)$, we determine the next configuration $C(t+1)$ as follows.  
For each input neuron $v$, the value of $firing^v$ in configuration $C(t+1)$ is determined as specified above for inputs, by an external source.
For each non-input neuron $v$, we compute the \emph{incoming potential} based on the firing state components of incoming neighbors $u$ of $v$, which we call $innbrs(v)$, in configuration $C(t)$ and the intervening edge weights, as:
\[
pot^v(t+1) = \Sigma_{u \in innbrs(v)} \ firing^u(t) \  weight(u,v).
\]
Then we set $firing^v(t+1)$ to $1$ if $pot^v(t+1) \geq threshold^v$, and $0$ otherwise.

An \emph{execution} of the network is a finite or infinite sequence $C(0), C(1), \ldots$, where $C(0)$ is an initial configuration, and each $C(t+1)$ is computed from $C(t)$ using the transition rule defined above.

\subsection{Examples}

We give three example abstract networks:  a line, a ring, and a hierarchy.  


\begin{example}
\label{ex: line1}
\emph{Line:}
The line graph $G = (V,E)$ is a line of $\lmax+1$ neurons,  numbered $0,1,2,\ldots,\lmax$.
Neuron $0$, which we imagine as the leftmost neuron, is the only input neuron.
The edges are all those of the form $(v,v+1)$, i.e., all the left-to-right edges between consecutively-numbered neurons.
All of the edge weights are $1$, and the thresholds for all the non-input neurons are $1$.
See Figure~\ref{fig: ex-figure: line1} for a depiction of the network, showing neuron numbers, thresholds and edge weights, for the special case where $\lmax = 5$.
\begin{figure}[t]
\vspace{-3cm}
\includegraphics[width = 6.5in]{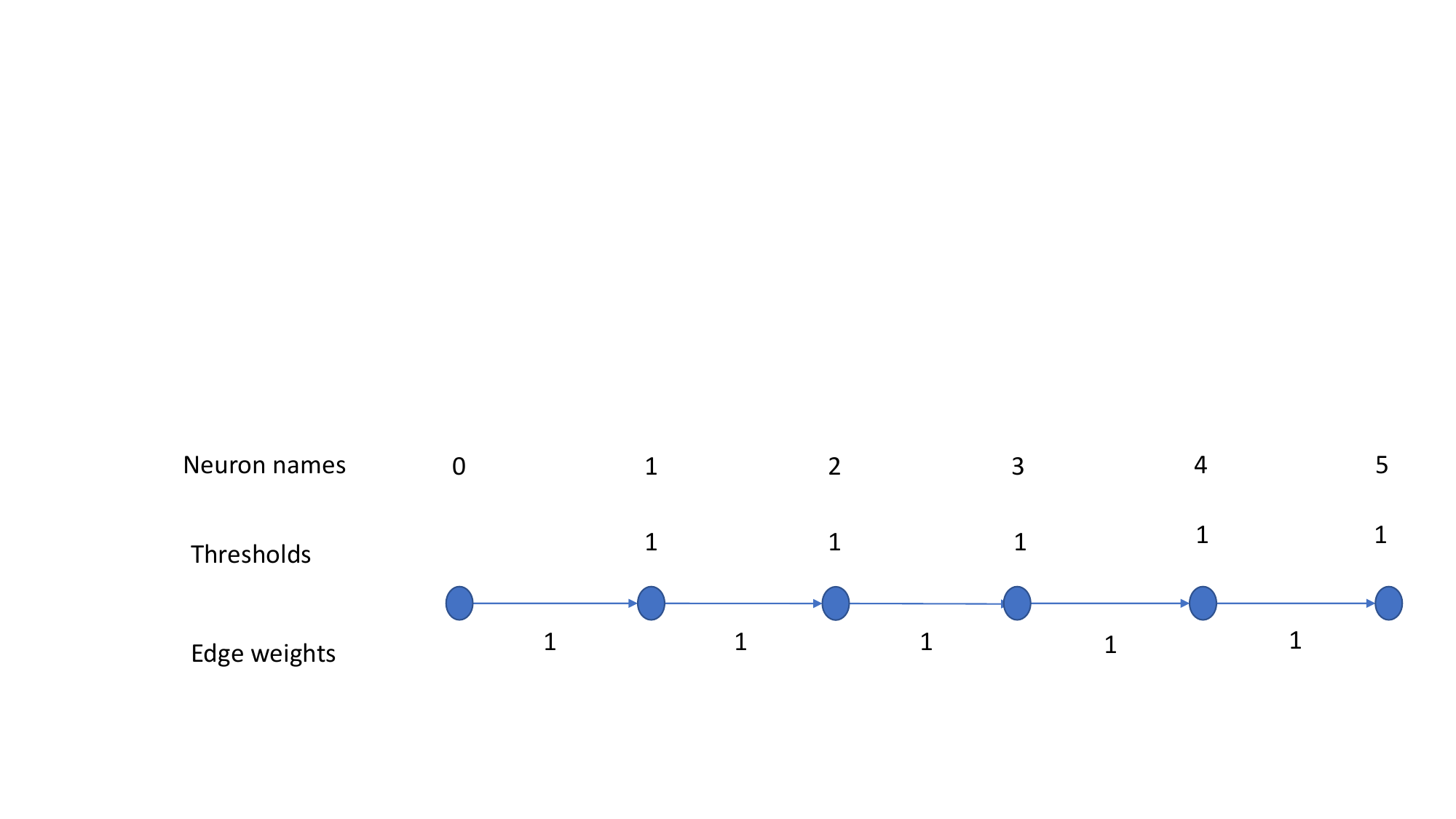}
\vspace{-2cm}
\caption{Line network, for $\lmax = 5$.}
\label{fig: ex-figure: line1}
\end{figure}

Every non-input neuron $v$ has its initial firing state component $firing^v(0)$ equal to $0$.
As for inputs, we assume that the input neuron $0$ fires at time $0$, and at no other time.
No other neurons fire at time $0$, because of the initialization condition just described.

The behavior of this network is as follows.  At each time $t$, $0 \leq t \leq \lmax$, neuron $t$ fires, and no other neuron fires.  
From time $\lmax+1$ onward, no neuron fires.
That is, one-time firing of the input neuron at time $0$ leads to a wave of firing that moves rightward across the line, one neuron per time step.
Each neuron fires just once.

Alternatively, we could assume that the input neuron $0$ fires at every time starting from time $0$.  In this case, for each time $t$, all neurons with numbers $\leq t$ will fire.  Thus, we get persistent firing for each neuron, once it gets started.

Another way of getting persistent firing is via self-loops on the non-input neurons.
For example, if we add just one self-loop, on neuron $1$, with weight $1$, then a single input pulse at time $0$ will lead to persistent firing for each neuron with number $\geq 1$, once it gets started.
That is, at each time $t$, all non-input neurons with numbers $\leq t$ will fire.
\end{example}

\begin{example}
\label{ex: ring1}
\emph{Ring:}
The ring graph $G = (V,E)$, consists of an input neuron, numbered $0$, plus a one-directional (clockwise) ring consisting of non-input neurons $1,2,\ldots,\lmax$.
The edges are all those of the form $(v,v+1)$, for $0 \leq v \leq \lmax-1$, plus the edge $(\lmax,1)$.  All of the edge weights are $1$, and the thresholds for all the non-input neurons are $1$.  See Figure~\ref{fig: ex-figure: ring1} for a depiction of the network, for the special case where $\lmax = 5$.

\begin{figure}[t]
\vspace{-1cm}
\includegraphics[width = 6.5in]{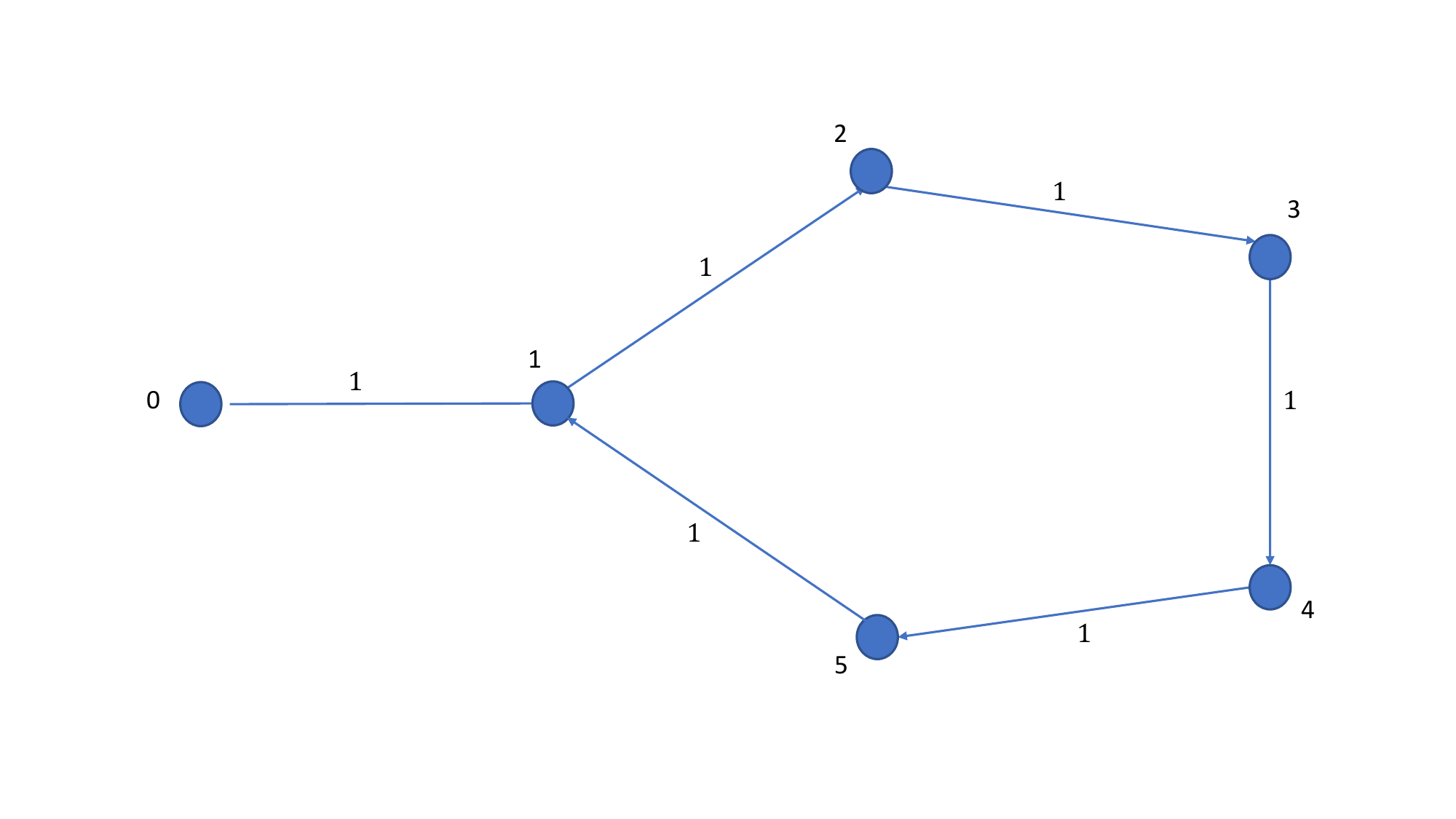}
\caption{Ring network, for $\lmax = 5$.  Neuron names and edge weights are shown.  Thresholds of non-input neurons are all $1$.}
\label{fig: ex-figure: ring1}
\end{figure}

Every non-input neuron has initial firing state component equal to $0$.
For inputs, we assume that the input neuron $0$ fires at time $0$, and at no other time. No other neurons fire at time $0$.

This network exhibits a periodic, circular firing pattern.
Neuron $0$ fires at time $0$ and at no other time. 
Each neuron $v$, $1 \leq v \leq \lmax$ fires at exactly times $t = v \bmod \lmax$.

As for the line example, continual input firing would lead to continual firing of the other neurons, once they get started.
Alternatively, adding a self-loop to neuron $1$ would achieve the same effect on neurons $1,2,\ldots,\lmax$.
\end{example}


\begin{example}  
\label{ex: hierarchy1}
\emph{Hierarchy:}
Hierarchy graphs were the only type of graph used in~\cite{lynch2024abstraction}.
Here we consider a directed graph $G = (V,E)$ that is a tree with $\lmax + 1$ levels, numbered $0,1,\ldots, \lmax$, with just one root neuron, $v_{\lambda}$, at the top level, level $\lmax$. 
Each non-leaf node has exactly $k$ children.
Thus, there are a total of $k^{\lmax}$ leaves.
The leaf neurons are the input neurons; all others are non-input neurons.

Edges are directed upwards, from child neurons to their parents, and have weight $1$.
The threshold for every non-input neuron is $r k$, where $r$ is a parameter, $0 < r \leq 1$.
This threshold is designed so that a non-input neuron will fire in response to the firing of at least an $r$-fraction of its children, thus, it fires in response to partial information.
Figure~\ref{fig: ex-figure: hierarchy1} contains a depiction of the network, in the special case where $\lmax = 3$, $k = 3$, and $r = 2/3$.
The thresholds are $r k = (2/3) 3 = 2$.

\begin{figure}[t]
\vspace{-1cm}
\includegraphics[width = 6.5in]{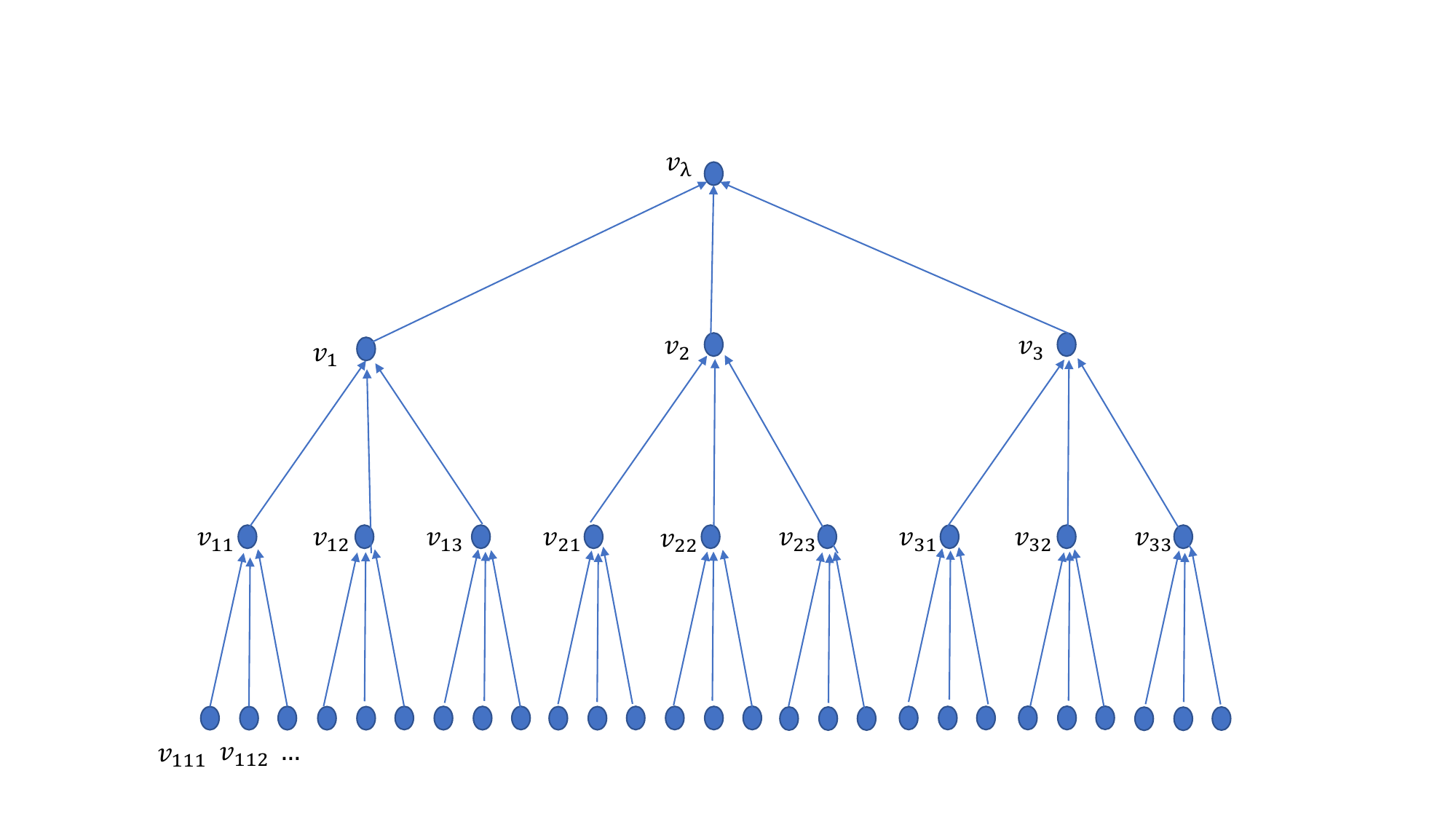}
\caption{Hierarchy network, for $\lmax = 3$ and $k = 3$.  Neuron names are shown.  Thresholds of non-input neurons are all $2$.  Edge weights are all $1$.}
\label{fig: ex-figure: hierarchy1}
\end{figure}

For inputs, we assume that some arbitrary subset $B$ of the leaf neurons fire at time $0$, and at no other time.
Then a wave of firing propagates upward, with a subset of the neurons in level $\ell$ firing at time $\ell$.  
The set of neurons that fire can be defined recursively:
a level $\ell$ neuron fires at time $\ell$ exactly if at least $r k$ of its children fire at time $\ell-1$.

In the special case shown in Figure~\ref{fig: ex-figure: hierarchy1}, suppose that $B$ consists of the eight neurons 
\[
v_{111}, v_{112}, v_{121}, v_{122}, v_{211}, v_{212}, v_{221}, v_{222}.
\]
The inputs in $B$ are enough to cause the root neuron $v_{\lambda}$ to fire at time $3$, based on firing of $v_{11}$, $v_{12}$, $v_{21}$, and $v_{22}$ at time $1$ and firing of $v_1$ and $v_2$ at time $2$.
This is because the threshold setting is $2$.
Although we have only eight inputs, they are placed in such a way as to satisfy this threshold for all the listed higher-level neurons.

On the other hand, we can consider much larger sets $B$ that will not cause the root neuron $v_{\lambda}$ to fire, such as:
\[
v_{111}, v_{112}, v_{113}, v_{121}, v_{122}, v_{123}, v_{131}, v_{132}, v_{133},
v_{211}, v_{212}, v_{213}, v_{221}, v_{231},
v_{311}, v_{312}, v_{313}, v_{321}, v_{331}. \\
\]
This is $19$ leaf neurons, but they are badly placed.
The only higher-level neurons that they cause to fire are 
\[
v_{11}, v_{12}, v_{13}, v_{21}, v_{31}, v_1.
\]


Once again, if we allow continual input firing, we get continual firing of the other neurons, once they get started.
Alternatively, we could add self-loops on the layer $1$ neurons, to achieve the same effect. 
This time, the self-loops would need to have weights of at least $r k$, in order to ensure that a layer $1$ neuron continues firing if it is triggered by its children to fire at time $1$.
\end{example}

Of course, many other examples are possible, since our abstract network model is a very general weighted digraph model.


\subsection{Detailed network model}
\label{sec: abstract-firing-model}

The detailed network model is the same as the abstract network model, with the only exception being that the detailed model admits some neuron and edge failures.  

Thus, the network digraph is arbitrary, except that input neurons have no self-loops.  Each neuron $v$ has a $firing$ state component, $firing^v$, and each non-input neuron $v$ has an associated $threshold$, $threshold^v$.  Each edge $e$ has an associated $weight$, $weight^e$.  Inputs are provided by an external source, by setting the $firing$ components of the input neurons.

We consider neuron failures in which neurons just stop firing and edge failure in which edges just stop transmitting signals.  For now, we consider only initial, permanent failures, which occur at the start of execution.  However, we think that the results should be extendable to other cases, such as the case of transient failures, where some neurons and edges are in a failed state at each time.

We let $F_V$ denote the set of failed vertices and $F_E$ the set of failed edges.
We let $S_V = V - F_V$ denote the set of non-failed, or \emph{surviving} vertices, and let $S_E = E - F_E$ denote the set of non-failed or \emph{surviving} edges.

Execution proceeds as for the abstract networks, with the exception that the nodes in $F_V$ never fire, i.e., $firing^v(t) = 0$ for every $t \geq 0$, and the edges in $F_E$ never transmit potential.
That is, for each non-input neuron $v$, we now compute the \emph{incoming potential} as: 
\[
pot^v(t+1) = \Sigma_{u \in innbrs(v): \ u \in S_V \ \wedge \ (u,v) \in S_E}  \ firing^u(t) \ weight(u,v),
\]
or equivalently,
\[
pot^v(t+1) = \Sigma_{u \in innbrs(v): \ (u,v) \in S_E}  \ firing^u(t) \ weight(u,v).
\]
These are equivalent because the neurons in $F_V$ do not affect the potential, since their $firing$ flags are always $0$.

For our definitions in Section~\ref{sec: correspondence} and results in Section~\ref{sec: mappings}, we will assume certain constraints on the numbers of failures.
These constraints should follow, with high probability, from assumptions about random failures.  However, we will avoid reasoning explicitly about probabilities in this paper.

\section{Corresponding Networks}
\label{sec: correspondence}

In this section, we describe corresponding abstract and detailed networks.  
Our starting point is an arbitrary abstract network, which we call $\mathcal{A}_1$.
From this, we derive a detailed network $\mathcal{D}$ in a systematic way.
$\mathcal{D}$ contains multiple copies of each neuron of $\mathcal{A}_1$.
This redundancy is intended to compensate for a limited number of neuron failures.
$\mathcal{D}$ also contains multiple copies of each edge of $\mathcal{A}_1$.
Namely, if $(u,v)$ is a directed edge in $\mathcal{A}_1$, then $\mathcal{D}$ contains an edge from each copy of $u$ to each copy of $v$. 
This redundancy is intended to compensate for a limited number of edge failures.
Thresholds for copies of neurons in $\mathcal{D}$ are somewhat reduced from thresholds of the original neurons in $\mathcal{A}_1$.

From abstract network $\mathcal{A}_1$, we also derive a second abstract network $\mathcal{A}_2$, by simply reducing the thresholds of $\mathcal{A}_1$, to the same level as the thresholds in $\mathcal{D}$.
This follows the approach of the earlier paper~\cite{lynch2024abstraction}.
There, instead of a single abstract algorithm $\mathcal{A}$ as one might expect, we gave two abstract algorithms, $\mathcal{A}_1$ and $\mathcal{A}_2$, one for proving firing guarantees and the other for proving non-firing guarantees.
Such a division seems necessary in view of the uncertainty in the detailed network $\mathcal{D}$.

For the rest of this section, fix the arbitrary abstract network $\mathcal{A}_1$.

In the rest of the paper, we generally use $u, v, w$ to denote neurons in abstract networks $\mathcal{A}_1$ and $\mathcal{A}_2$, and $x, y, z$ to denote neurons in detailed network $\mathcal{D}$.

\subsection{Definition of $\mathcal{D}$}
\label{sec: def-D}

Here we describe how to construct the detailed network $\mathcal{D}$ directly from the given abstract network $\mathcal{A}_1$.
$\mathcal{D}$ uses a positive integer parameter $m$.
For each neuron $v$ of $\mathcal{A}_1$, $\mathcal D$ contains exactly $m$ corresponding neurons, which we call $copies$ of $v$; we denote the set of copies of $v$ by $copies(v)$.


As for the edges, if $(u,v)$ is an edge in $\mathcal{A}_1$, $x \in copies(u)$ and $y \in copies(v)$, then we include an edge $(x,y)$ in $\mathcal{D}$.
These are the only edges in $\mathcal{D}$.
We define the weight of each such edge $(x,y)$, $weight(x,y)$, to be 
\[
weight(x,y) = \frac{ weight(u,v)}{m}.
\]
That is, we divide the weights of edges of $\mathcal{A}_1$ by $m$.\footnote{The approach here differs slightly from that in~\cite{lynch2024abstraction}.  There, we kept the edge weights unchanged but increased the threshold by a factor of $m$.  The new approach seems a bit simpler and more natural.
}

We must also define the thresholds for the non-input neurons of $\mathcal{D}$.
Suppose that the threshold for a particular neuron $v$ in $\mathcal{A}_1$ is $h$, i.e.,
$threshold^v = h$, and consider any $y \in copies(v)$.

The threshold for $v$, $threshold^v$, is supposed to accommodate the potential from all the incoming neighbors of $v$.
The threshold for $y$, $threshold^y$, is supposed to accommodate the potential from all copies of all the incoming neighbors of $v$.
For each incoming neighbor $u$ of $v$, there are $m$ copies, but on the other hand, the weights of edges from copies of $u$ to $y$ are divided by $m$, compared to the corresponding weight of edge $(u,v)$ in $\mathcal{A}_1$.
So that might suggest that $threshold^y$ should be the same as $threshold^v$, which is $h$.

However, we also must take into account possible failures of neurons and edges.
For this purpose, we introduce new factors of $s_V$, $0 < s_V \leq 1$, for survival of neurons, and $s_E$, $0 < s_E \leq 1$, for survival of edges.\footnote{In~\cite{lynch2024abstraction}, we used other notation:  $1 - \epsilon$ for $s_V$ and $a$ for $s_E$.  The notation in this paper seems more natural, since it emphasizes the parallel between the failures of neurons and edges.}
In terms of this notation, we define $threshold^y = s_V s_E h$.

In an execution of $\mathcal D$, some of the neurons and edges may fail, but we impose two constraints on the failure sets $F_V$ and $F_E$ (and their complementary survival sets $S_V$ and $S_E$).
As noted above, we consider here just initial failures, of both neurons and edges.
The constraints are:
\begin{enumerate}
\item 
For each neuron $v$ in $\mathcal{A}_1$, at least $s_V m$ of the neurons in $copies(v)$ survive, i.e., do not fail.  That is, there are at least $s_V m$ copies of $v$ in the survival set $S_V$.
\item
For each neuron $v$ in $\mathcal{A}_1$, each edge $(u,v)$ in $\mathcal{A}_1$, and
each $y \in copies(v)$, there are at least $s_V s_E m$ surviving edges to $y$ from copies of $u$ that survive.
That is, there are at least $s_V s_E m$ edges in $S_E$ from copies of $u$ in $S_V$.
\end{enumerate}

Now we define corresponding executions of $\mathcal{A}_1$ and $\mathcal{D}$.
For this, we first fix the failures and survival sets $F_V$, $F_E$, $S_V$, and $S_E$ for $\mathcal D$. 
Then all we need to do to define corresponding executions of $\mathcal{A}_1$ and $\mathcal{D}$ is to specify the inputs to the two networks; the executions are then completely determined.

So consider any arbitrary pattern of inputs to $\mathcal{A}_1$; this means a specification of which input neurons fire in $\mathcal{A}_1$ at each time $t$.
Then for every input neuron $v$ of $\mathcal{A}_1$ and every time $t$:
\begin{enumerate}
\item 
If $v$ fires at time $t$ in $\mathcal{A}_1$ then all of the neurons in $copies(v)$ that survive fire in $\mathcal{D}$.  (By assumption on the network model, none of the failed neurons fire.)
\item
If $v$ does not fire at time $t$ in $\mathcal{A}_1$ then none of the neurons in $copies(v)$ fire in $\mathcal{D}$.   
\end{enumerate}


\begin{example}
\emph{Line:}
\label{ex: line2}
We consider the detailed network corresponding to the abstract network of Example~\ref{ex: line1}.  
Figure~\ref{fig: ex-figure: line2} depicts a special case of that network, with $\lmax = 5$, as in Figure~\ref{fig: ex-figure: line1}, and with $m = 4$, $s_V = 3/4$, and $s_E = 2/3$. 
Thresholds for non-input neurons are $s_V s_E h = (3/4)(2/3) 1 = 1/2$.  
Edge weights are $1/4$.

\begin{figure}[t]
\vspace{-1cm}
\includegraphics[width = 6.5in]{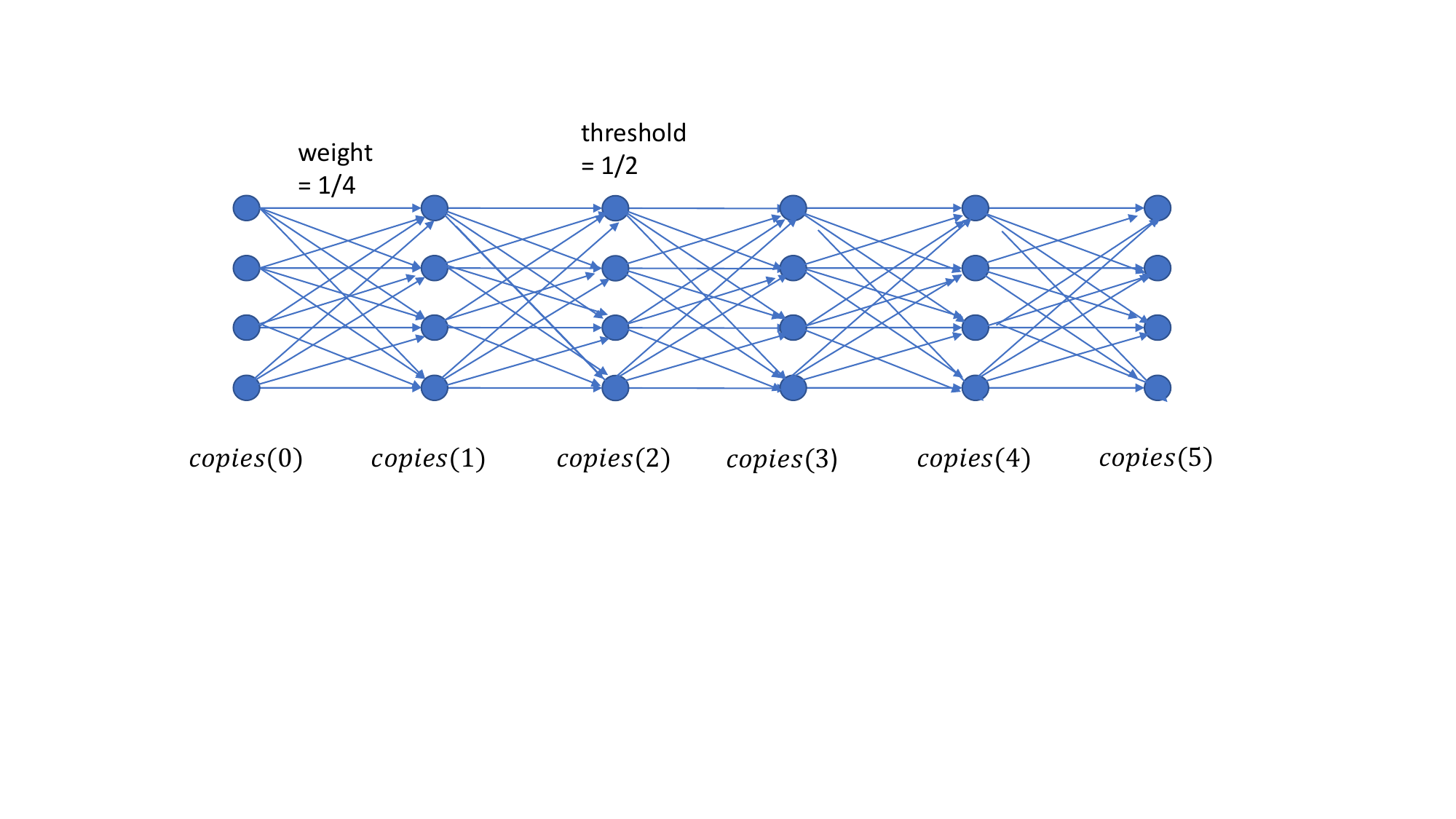}
\caption{Detailed line network, for $\lmax = 5$, $m=4$, thresholds are $1/2$, edge weights are $1/4$.}
\label{fig: ex-figure: line2}
\end{figure}

Assuming that the network satisfies the two constraints in Section~\ref{sec: def-D}, one possible failure pattern for this special-case network might be as follows:
For every neuron $v$ of $\mathcal{A}_1$, the highest-numbered neuron in $copies(v)$ fails (according to some arbitrary numbering).
Also, for every edge $(u,v)$ of $\mathcal{A}_1$ and every $y \in copies(v)$, the edge from the lowest-numbered neuron in $copies(u)$ to $y$ fails.

Assume that the input arrives at time $0$ only.
Then the first three neurons in $copies(0)$ fire at time $0$.
Thereafter, for each time $v \geq 1$, each of the first three neurons $y \in copies(v)$ has its threshold of $1/2$ met for time $v$ by the firing of surviving neurons in $copies(v-1)$ whose edges to $y$ also survive.
So the first three neurons $y \in copies(v)$ fire at time $v$.

In this way, the behavior of the detailed network can be viewed as imitating the behavior of the abstract network.
We will make the correspondence between behaviors more precise in Section~\ref{sec: mappings}.
\end{example}

\begin{example}
\label{ex: ring2}
\emph{Ring:}
We consider the detailed network corresponding to the abstract network of Example~\ref{ex: ring1}.  As in the previous example, we consider the special case where $\lmax = 5$, $m=4$, $s_V = 3/4$, and $s_E = 2/3$. 
Thresholds in $\mathcal{D}$ are again $1/2$ and edge weights are $1/4$.

We consider the same failure pattern as in the previous example:  
For every neuron $v$ of $\mathcal{A}_1$, the highest-numbered neuron in $copies(v)$ fails.
Also, for every edge $(u,v)$ of $\mathcal{A}_1$ and every $y \in copies(v)$, the edge from the lowest-numbered neuron in $copies(u)$ to $y$ fails.

Assume that the input arrives at time $0$ only.  
Then the first three neurons in $copies(0)$ fire at time $0$.  
Thereafter, each of the first three neurons in $copies(v)$, $v \geq 1$, fire at every time $t$ such that $t = v \bmod 5$.
Again, the behavior of the detailed network can be viewed as imitating the behavior of the abstract network.
\end{example}

\begin{example}
\emph{Hierarchy:}
\label{ex: hierarchy2}
Finally, we consider the detailed network corresponding to the abstract network of Example~\ref{ex: hierarchy1}.
Now we consider the special case where $\lmax = 3$, $k=3$, $r = 2/3$, $m=4$, $s_V = 3/4$, and $s_E = 2/3$. 
So the thresholds in $\mathcal{D}$ are now $s_V s_E rk = (3/4)(2/3)(2/3)3 = 1$, reduced from $2$ in $\mathcal{A}_1$.
Edge weights are $1/4$.

First, we consider the same failure pattern as in the previous two examples:
For every neuron $v$ of $\mathcal{A}_1$, the highest-numbered neuron in $copies(v)$ fails.
Also, for every edge $(u,v)$ of $\mathcal{A}_1$ and every $y \in copies(v)$, the edge from the lowest-numbered neuron in $copies(u)$ to $y$ fails.

Assume that the input arrives at time $0$ only.
Suppose the input set $B$ is the first input set described in Example~\ref{ex: hierarchy1}, namely, 
\[
v_{111}, v_{112}, v_{121}, v_{122}, v_{211}, v_{212}, v_{221}, v_{222}.
\]
Then the first three neurons in $copies(v)$ for each $v \in B$ fire at time $0$.
Thereafter, for every neuron $u$ that fires at time $t$ in the execution of the abstract network on input $B$, each of the first three neurons in $copies(v)$ fires at time $t$ in the detailed network.
This can be seen by arguing recursively on the levels of the network, using the particular failure pattern and the new thresholds of $1$.
Once again, the behavior of the detailed network can be viewed as imitating the behavior of the abstract network.

On the other hand, in this example, unlike the previous two examples, the low thresholds can introduce some new behaviors that do not imitate the behavior of the abstract network.
For example, consider a situation in which no failures occur, and consider the set $B$ that consists only of the single neuron $v_{111}$.
Assume the input arrives at time $0$, as usual.
Then all four of the copies of $v_{111}$ fire, contributing incoming potential of $4 (1/4) = 1$ to each copy of $v_{11}$.  This is enough to meet the thresholds of $1$ for these copies.  So all four copies of $v_{11}$ fire at time $1$.
In the same way, we see that all four copies of $v_{1}$ fire at time $2$, and all four copies of the root neuron $v_{\lambda}$ fire at time $3$.
Thus, the firing of only a single input neuron in the abstract network is enough to cause firing of all four copies of the root neuron in the detailed network.
This does not imitate any behavior of the abstract network.

Identifying the conditions under which the detailed network does behave like the abstract network is the subject of Section~\ref{sec: mappings}.
\end{example}

\subsection{Definition of $\mathcal{A}_2$}
\label{sec: def-A2}

Now we define the second abstract network $\mathcal{A}_2$.
The two networks $\mathcal{A}_1$ and $\mathcal{A}_2$ have the same sets of neurons and edges, and the same edge weights.
The only difference is in the thresholds of the non-input neurons, which are reduced in $\mathcal{A}_2$ in order to accommodate failed neurons and edges.

Namely, assume we are given constants $s_V$, $0 < s_V \leq 1$, and $s_E$, $0 < s_E \leq 1$, as in the definition of $\mathcal{D}$.
In fact, we use the same values of $s_V$ and $s_E$ in constructing $\mathcal{A}_2$ as we did for $\mathcal{D}$.
Then if $threshold^v = h$ in $\mathcal{A}_1$, then $threshold^v$ in $\mathcal{A}_2$ is reduced by multiplying $h$ by $s_V s_E$.  That is, $threshold^v = s_V s_E h$ in $\mathcal{A}_2$.
Note that this is the same as the thresholds of the copies of $v$ in $\mathcal{D}$.

Now we define corresponding executions of $\mathcal{A}_2$ and $\mathcal{D}$.  The method is the same as for $\mathcal{A}_1$ and $\mathcal{D}$.
Namely, we first fix the failure and survival sets $F_V$, $F_E$, $S_V$, and $S_E$ for $\mathcal{D}$,
Then we specify the inputs to the two networks; the executions are then determined.

So consider any arbitrary pattern of inputs to $\mathcal{A}_2$.
Then for every input neuron $v$ of $\mathcal{A}_2$ and every time $t$:
\begin{enumerate}
\item 
If $v$ fires at time $t$ in $\mathcal{A}_2$ then all of the neurons in $copies(v)$ that survive fire in $\mathcal{D}$.  (By assumption on the network model, none of the failed neurons fire.)
\item
If $v$ does not fire at time $t$ in $\mathcal{A}_2$ then none of the neurons in $copies(v)$ fire in $\mathcal{D}$.   
\end{enumerate}

What is the difference between the executions of $\mathcal{A}_1$ and $\mathcal{A}_2$ on the same input patterns?
The only difference between the structure of these two networks is the somewhat lower thresholds in $\mathcal{A}_2$. 
Theoretically, this could mean that certain neurons might fire at certain times in the execution of $\mathcal{A}_2$ but not in the execution of $\mathcal{A}_1$.
But depending on the values of the various parameters, this difference might or might not actually arise. 

For instance, in the three Examples~\ref{ex: line1}, \ref{ex: ring1}, and~\ref{ex: hierarchy1}, the firing behavior would be exactly the same for $\mathcal{A}_1$ and $\mathcal{A}_2$, on the same input pattern, provided that the value of $s_V s_E$ is large enough.
In fact, $s_V s_E > 0$ is good enough in Examples~\ref{ex: line1} and~\ref{ex: ring1}, and $s_V s_E > 1/2$ is good enough in Example~\ref{ex: hierarchy1}.
The reason is that the differences in the thresholds here are small with respect to the differences in potential caused by the firing of different numbers of incoming neighbors.

On the other hand, if the value of $s_V s_E$ is small, then new firing behaviors can be introduced in $\mathcal{A}_2$ that do not occur in $\mathcal{A}_1$.
For instance, in the hierarchy example, if $s_V s_E = 1/2$, then the thresholds in $\mathcal{A}_2$ are reduced to $1$, and firing of the single input neuron $v_{111}$ is enough to trigger firing of higher-level neurons $v_{11}$, $v_1$, and $v_{\lambda}$.
This firing behavior does not happen in $\mathcal{A}_1$.

\section{Mapping Theorems}
\label{sec: mappings}

This section contains our main results.
We assume two abstract networks $\mathcal{A}_1$ and $\mathcal{A}_2$, and detailed network $\mathcal{D}$, satisfying the correspondence definitions in Section~\ref{sec: correspondence}. 
We assume that $\mathcal{D}$ and $\mathcal{A}_2$ use the same values of $s_V$ and $s_E$.
We give three \emph{mapping theorems}, showing precise, step-by-step relationships between the behavior of $\mathcal{D}$ and the behavior of the two abstract networks.

The first theorem, Theorem~\ref{thm: firing}, says that, if a neuron $v$ of $\mathcal{A}_1$ fires at a certain time, then at least $s_V m$ of the copies of $v$ in $\mathcal{D}$ also fire at that time.
The second theorem, Theorem~\ref{thm: non-firing}, says that, if a neuron $v$ of $\mathcal{A}_2$ does not fire at a certain time, then none of the copies of $v$ in $\mathcal{D}$ fire at that time.
Together, these two results say that, if the networks correspond as described in Section~\ref{sec: correspondence}, then the firing behavior of $\mathcal{D}$ is guaranteed to faithfully follow the firing behavior of the abstract networks.

Thus, if the input to the abstract networks is sufficient to ensure firing of a neuron $v$ in $\mathcal{A}_1$, where the thresholds are high, then most of the copies 
of $v$ will fire in $\mathcal{D}$, whereas if the input is not sufficient to ensure firing of $v$ in $\mathcal{A}_2$, where the thresholds are low, then none of the copies of $v$ will fire in $\mathcal{D}$.
There is a middle ground, where the input to the abstract networks is sufficient to ensure firing of $v$ when the thresholds are low, but not sufficient to ensure firing when the thresholds are high.
In this case, we do not make any claims about the firing behavior of the copies of $v$ in $\mathcal{D}$.
Indeed, it would be hard to make any such claims, because of the uncertainty in the behavior of $\mathcal{D}$ resulting from the possible failures of neurons and edges.
The discussion in Example~\ref{ex: hierarchy2} should give some idea of the issues.

These results and their proofs generally follow the approach in~\cite{lynch2024abstraction}. 
In particular, in~\cite{lynch2024abstraction}, we also show mappings between a detailed network $\mathcal{D}$ and two abstract networks $\mathcal{A}_1$ and $\mathcal{A}_2$.

We also include a third mapping theorem, Theorem~\ref{thm: actuator}.  This theorem compares the external effects of the three corresponding networks when they are used to trigger a reliable external "actuator" neuron.

\subsection{Mapping between $\mathcal{A}_1$ and $\mathcal{D}$}

We consider corresponding executions of $\mathcal{A}_1$ and $\mathcal{D}$, beginning with corresponding inputs, as defined in Section~\ref{sec: correspondence}.
Namely, we fix sets $F_V$ and $F_E$ of failures, and their complements $S_V$ and $S_E$ of survivals.
We fix inputs of $\mathcal{A}_1$ and derive the corresponding inputs of $\mathcal{D}$ as in Section~\ref{sec: correspondence}.
This uniquely determines the executions of $\mathcal{A}_1$ and $\mathcal{D}$.
Theorem~\ref{thm: firing} refers to these executions.

\begin{theorem}
\label{thm: firing}
Consider any neuron $v$ in abstract network $\mathcal{A}_1$.
At any time during the executions of $\mathcal{A}_1$ and $\mathcal{D}$,
if neuron $v$ fires, then at least $s_V m$ of the neurons in $copies(v)$ fire.
\end{theorem}

\begin{proof}
We proceed by induction on the times $t$ during the execution, to prove the following slightly stronger statement $P(t)$, for every time $t$:  \\
$P(t)$:  For every neuron $v$ of $\mathcal{A}_2$, if $v$ fires at time $t$ in the execution of $\mathcal A_1$, then every surviving neuron in $copies(v)$ fires at time $t$ in the execution of $\mathcal D$. \\
This statement implies the theorem statement because, by assumption, for any $v$, at least $s_V m$ of the neurons in $copies(v)$ survive.

\emph{Base:}  Time $0$. \\
Suppose that some neuron $v$ fires at time $0$ in the execution of $\mathcal A_1$.
Then $v$ is an input neuron that is selected for firing.  By the input conventions for the neurons in $copies(v)$, all of the surviving copies fire at time $0$, as needed.

\emph{Inductive step:}  Time $t+1$. \\
Suppose that some neuron $v$ fires at time $t+1$.  Then some collection of incoming neighbor neurons fire at time $t$, sufficient to meet $v$'s threshold, say $h$.
Specifically, the sum of the weights of the edges incoming to $v$ from the set of incoming neighbors of $v$ that fire at time $t$ is at least $h$.
That is, 
\[\sum_{u} weight(u,v) \geq h,
\]
where $u$ ranges over the set of incoming neighbors to $v$ in $\mathcal{A}_1$.

For each of the firing incoming neighbors $u$ of $v$, by inductive hypothesis $P(t)$, every surviving neuron in $copies(u)$ fires at time $t$ in the execution of $\mathcal D$.

Now consider any particular surviving neuron $y \in copies(v)$; we show that $y$ fires at time $t+1$ in the execution of $\mathcal D$, as needed.
Consider any particular incoming neighbor $u$ of $v$ in $\mathcal{A}_1$.
By an assumption on $\mathcal D$, we know that there are at least $s_V s_E m$ neurons in $copies(u)$ that both survive and also have surviving edges to $y$.
Since these neurons all survive, by the inductive hypothesis $P(t)$, they all fire at time $t$.

Thus, there are at least $s_V s_E m$ firing neurons in $copies(u)$ that have surviving edges to $y$.
By definition of the weights in $\mathcal{D}$, each such neuron $x$ is connected to $y$ with an edge with weight $\frac{weight(u,v)}{m}$.
So this entire collection of firing neurons in $copies(u)$ together contributes at least 
\[
s_V s_E m \frac{weight(u,v)}{m} = s_V s_E weight(u,v)
\]
potential to $y$.

Considering all the incoming neighbors $u$ of $v$ together, we get a total incoming potential to $y$ of at least 
\[\sum_u s_V s_E \ weight(u,v) = s_V s_E \sum_u weight(u,v) \geq s_V s_E h,
\]
which is enough to meet the firing threshold for $y$ in $\mathcal{D}$.
Therefore, since $y$ survives, it fires at time $t+1$, as needed.
\end{proof}

\begin{example}
\label{ex: line3}
\emph{Line:}
We revisit the line example from Examples~\ref{ex: line1} and~\ref{ex: line2}.
Let $\mathcal{A}_1$ be the abstract network from Example~\ref{ex: line1}, with $\lmax = 5$, and let $\mathcal{D}$ be the corresponding detailed network from Example~\ref{ex: line2}, with $m = 4$, $s_V = 3/4$, and $s_E = 2/3$.
Suppose that, in $\mathcal{A}_1$, the input neuron fires just at time $0$, and the inputs in $\mathcal{D}$ correspond.
We know that, in $\mathcal{A}_1$, each neuron $v$ fires at exactly time $v$.
Then Theorem~\ref{thm: firing} implies that, at any time $v \geq 0$, at least $s_V 4 = (3/4) 4 = 3$ of the neurons in $copies(v)$ fire.

Now let $\mathcal{A}_1$ be the line example from Example~\ref{ex: line1}, with an arbitrary value of $\lmax$.
Now suppose that the input neuron $0$ fires at all even-numbered times $0, 2, 4, \ldots$.
Then it is easy to see that, in $\mathcal{A}_1$, each neuron $v$ fires at exactly times $v, v+2, v+4, \ldots$.
For example, at time $7$, exactly neurons $1$, $3$, $5$, and $7$ fire.
Let $\mathcal{D}$ be the corresponding detailed network, and suppose that $m = 4$, $s_V = 3/4$, and $s_E - 2/3$.
Then Theorem~\ref{thm: firing} implies that, for any neuron $v$, at least $s_V 4 = (3/4) 4 = 3$ copies of $v$ fire at each time $v, v+2, v+4, \ldots$.
\end{example}

\begin{example}
\label{ex: hierarchy3}
\emph{Hierarchy:}
We revisit the hierarchy example from Examples~\ref{ex: hierarchy1} and~\ref{ex: hierarchy2}.
Let $\mathcal{A}_1$ be the abstract network from Example~\ref{ex: hierarchy1}, with $\lmax = 3$, $k = 3$, and $r = 2/3$, and let $\mathcal{D}$ be the corresponding detailed network from Example~\ref{ex: hierarchy2}, with $m = 4$, $s_V = 3/4$, and $s_E = 2/3$.
Let the input set $B$ for $\mathcal{A}$ be
\[
v_{111}, v_{112}, v_{121}, v_{122}, v_{211}, v_{212}, v_{221}, v_{222},
\]
and suppose that the neurons in $B$ fire at time $0$ in $\mathcal{A}_1$, and the inputs in $\mathcal{D}$ correspond.
We know that, in $\mathcal{A}_1$, neuron $v_{\lambda}$ fires at time $3$.
Then Theorem~\ref{thm: firing} implies that, in $\mathcal{D}$, at least $s_V m = (3/4) 4 = 3$ copies of $v_{\lambda}$ fire at time $3$.

Now let $\mathcal{A}_1$ be the abstract network from Example~\ref{ex: hierarchy1}, this time with $\lmax = 3$, $k = 5$, and $r = 4/5$.
Let $\mathcal{D}$ be the corresponding detailed network from Example~\ref{ex: hierarchy2}, this time with more copies of each neuron of $\mathcal{A}_1$ and larger values of the survival parameters $s_V$ and $s_E$.
Specifically, let $m = 32$, $s_V = 15/16$, and $s_E = 14/15$.
For this case,
Theorem~\ref{thm: firing} implies that, in $\mathcal{D}$, at least $s_V m = (15/16) 32 = 30$ copies of $v_{\lambda}$ fire at time $3$.  
\end{example}

\subsection{Mapping between $\mathcal{A}_2$ and $\mathcal{D}$}

Now we consider corresponding executions of $\mathcal{A}_2$ and $\mathcal{D}$, beginning with corresponding inputs.  
Again, we fix sets $F_V$, $F_E$, $S_V$, and $S_E$, and the inputs of $\mathcal{A}_2$ and corresponding inputs of $\mathcal{D}$. 
This yields unique executions of $\mathcal{A}_2$ and $\mathcal{D}$.
Theorem~\ref{thm: non-firing} refers to these executions.

\begin{theorem}
\label{thm: non-firing}
Consider any neuron $v$ in abstract network $\mathcal{A}_2$.
At any time during the executions of $\mathcal{A}_2$ and $\mathcal{D}$, if neuron $v$ does not fire, then none of the neurons in $copies(v)$ fire.
\end{theorem}

\begin{proof}
We proceed by induction on the times $t$ during the execution, to prove the following statement $Q(t)$, for every time $t$: \\
$Q(t)$:  For every neuron $v$ of $\mathcal{A}_2$, if $v$ does not fire at time $t$ in the execution of $\mathcal{A}_2$, then no neurons in $copies(v)$ fire at time $t$.

\emph{Base:}  Time $0$. \\
Suppose that some neuron $v$ does not fire at time $0$ in the execution of $\mathcal A_2$.  Then by the conventions for the neurons in $copies(v)$, none of these copies fire at time $0$.
This is as needed.

\emph{Inductive step:}  Time $t+1$. \\
Suppose that some neuron $v$ does not fire at time $t+1$ in the execution of $\mathcal{A}_2$.
Then there is insufficient firing of incoming neighbors at time $t$ to meet $v$'s threshold in $\mathcal{A}_2$, which is $s_V s_E h$, where $threshold^v = h$ in $\mathcal{A}_1$.
Specifically, the sum of the weights of the edges incoming to $v$ from the set of incoming neighbors of $v$ that fire at time $t$ is strictly less than $s_V s_E h$.
That is, 
\[
\sum_{u} weight(u,v) < s_V s_E h,
\]
where $u$ ranges over the set of incoming neighbors to $v$ in $\mathcal{A}_2$.

For each incoming neighbor $u$ of $v$ that does not fire at time $t$, by inductive hypothesis $Q(t)$, no neuron in $copies(u)$ fires at time $t$ in the execution of $\mathcal{D}$.

Now consider any particular neuron $y \in copies(v)$; we show that $y$ does not fire at time $t+1$ in the execution of $\mathcal{D}$, as needed.
To show this, we argue that the total potential incoming to $y$ for time $t+1$ is strictly less than its threshold $s_V s_E h$.

We consider two kinds of incoming neighbors that might contribute potential to $y$.
First, consider the neurons in $copies(u)$, where $u$ is a neuron that does not fire at time $t$ in $\mathcal{A}_2$.
By inductive hypothesis $Q(t)$, we know that none of these neurons fire in $\mathcal{A}_2$ at time $t$ and so they do not contribute any potential to $y$ for time $t+1$.

Second, consider the neurons in $copies(u)$, where $u$ is a neuron that does fire at time $t$ in $\mathcal{A}_2$.
The largest possible potential that could arise is for the case where all the copies of all the firing incoming neighbors $u$ of $v$ survive, and also their edges to $y$ survive.
Even in this extremely favorable case, it turns out that the total incoming potential to $v$ for time $t+1$ is the same as the incoming potential to $v$ in $\mathcal{A}_2$, which is $\sum_{u} weight(u,v)$.

To see this, write this extreme potential as a double summation:
\[
\sum_{u} \sum_{x \in copies(u)} weight(x,y),
\]
where $u$ ranges over the set of incoming neighbors to $v$ that fire at time $t$ in $\mathcal{A}_2$.
This is equal to 
\[
\sum_{u} \sum_{x \in copies(u)} (\frac{weight(u,v)}{m}),
\]
because of the way the weights are defined in $\mathcal{D}$.
This is in turn equal to 
\[\sum_{u} m (\frac{weight(u,v)}{m}) =\sum_{u} weight(u,v),
\]
where, as noted already, $u$ ranges over the set of incoming neighbors to $v$ that fire at time $t$ in $\mathcal{A}_2$.

But we already know that 
\[
\sum_{u} weight(u,v) < s_V s_E h.
\]
Since the right-hand side of this inequality is $threshold^y$ in $\mathcal{D}$,
$y$ does not receive enough incoming potential to fire at time $t+1$ in $\mathcal{D}$, as needed.
\end{proof}

\begin{example}
\label{ex: line4}
\emph{Line:} 
Let $\mathcal{A}_1$ be the line example from Example~\ref{ex: line1}, with arbitrary $\lmax$, and obtain $\mathcal{A}_2$ by lowering the thresholds to  $s_V s_E 1$, where $s_V = 3/4$ and $s_E = 2/3$, which is $1/2$.
Consider the corresponding network $\mathcal{D}$ with $m = 4$, $s_V = 3/4$, and $s_E = 2/3$.
Then Theorem~\ref{thm: non-firing} implies that non-firing of any neuron $v$ at any time $t$ in $\mathcal{A}_2$ implies non-firing of any of its copies at time $t$ in $\mathcal{D}$.

Moreover, as noted at the end of Section~\ref{sec: def-A2}, for this case, the firing behavior is exactly the same for $\mathcal{A}_1$ and $\mathcal{A}_2$, on the same input pattern. 
Then we can combine the results from Theorems~\ref{thm: firing} and~\ref{thm: non-firing} to get a strong combined claim, involving just $\mathcal{A}_1$ and $\mathcal{D}$.
Namely, consider any fixed input pattern.
Then for any neuron $v$ in $\mathcal{A}_1$ and any time $t$, 
firing of $v$ at time $t$ in $\mathcal{A}_1$ implies firing of at least $(3/4) m$ of its copies at time $t$ in $\mathcal{D}$.
And non-firing of $v$ at time $t$ in $\mathcal{A}_1$ implies non-firing of any of its copies at time $t$ in $\mathcal{D}$.
This expresses a very close correspondence between the behavior of $\mathcal{A}_1$ and $\mathcal{D}$.
\end{example}

\begin{example}
\label{ex: hierarchy4}
\emph{Hierarchy:}
Let $\mathcal{A}_2$ be the hierarchy example from Example~\ref{ex: hierarchy1}, with arbitrary $\lmax$, $k$ and $r$.
Thresholds are $r k$, and edge weights are $1$.

In Example~\ref{ex: hierarchy2} we saw how small values of $s_V$ and $s_E$ can introduce firing patterns in the detailed network $\mathcal{D}$ that do not correspond to firing patterns in $\mathcal{A}_1$.
For example, with $k = 3$ and $r = 2/3$, the choice of $s_V = 3/4$ and $s_E = 2/3$ can cause copies of $v_{\lambda}$ to fire in $\mathcal{D}$ when only copies of $v_{111}$ are provided as input.
Theorem~\ref{thm: non-firing} still applies in this case, of course, but all it says is that, if a neuron $v$ doesn't fire in $\mathcal{A}_2$, then its copies don't fire in $\mathcal{D}$.
For example, if $v_\lambda$ doesn't fire in $\mathcal{A}_2$ then its copies don't fire in $\mathcal{D}$.
But this is not saying very much, because $v_{\lambda}$ can fire in $\mathcal{A}_2$ in a wide range of situations, because $\mathcal{A}_2$ has the same low thresholds as $\mathcal{D}$.

One way to avoid this phenomenon is to place enough constraints on parameters to ensure that $\mathcal{A}_1$ and $\mathcal{A}_2$ exhibit the same firing behavior on any inputs.
For example, suppose that $k = 5$ and $r = 4/5$, so $r k = 4$.
Suppose that $s_V = 15/16$, and $s_E = 14/15$.
Then the thresholds in $\mathcal{A}_1$ are $rk = 4$, while the thresholds in $\mathcal{A}_2$ are $s_V s_E r k = (15/16)(14/15) 4 = 7/2$.
Since there is no integer $n$ such that $7/2 \leq n < 4$, this means that exactly the same sets of firing children trigger the firing of corresponding neurons in $\mathcal{A}_1$ and $\mathcal{A}_2$.
This implies that the behavior of $\mathcal{A}_1$ and $\mathcal{A}_2$ is the same on all inputs.

Then we consider $\mathcal{D}$, with arbitrary $m$, say $m = 32$.
$\mathcal{D}$ has the same threshold, $7/2$, as $\mathcal{A}_2$.
Theorem~\ref{thm: non-firing} implies that non-firing of any $v$ at any time $t$ in $\mathcal{A}_2$ implies non-firing of any of its copies at time $t$ in $\mathcal{D}$.
As in the previous example, we can combine the firing and non-firing results to get a strong combined claim, involving just $\mathcal{A}_1$ and $\mathcal{D}$.
Namely, consider any fixed input pattern.
Then for any neuron $v$ in $\mathcal{A}_1$ and any time $t$, 
firing of $v$ at time $t$ in $\mathcal{A}_1$ implies firing of at least $(15/16) 32 = 30$ of its copies at time $t$ in $\mathcal{D}$.
And non-firing of $v$ at time $t$ in $\mathcal{A}_1$ implies non-firing of any of its copies at time $t$ in $\mathcal{D}$.
This expresses a very close correspondence between the behavior of $\mathcal{A}_1$ and $\mathcal{D}$.
\end{example}


Notice that, in Example~\ref{ex: line4} and the last part of Example~\ref{ex: hierarchy4}, $\mathcal{A}_1$ and $\mathcal{A}_2$ behave exactly the same, on any choice of inputs.  This is not always the case---for example, in the hierarchy example, if the number $k$ of children is very large, then some inputs could cause some neurons to fire because they meet the lower thresholds of $\mathcal{A}_2$ but not the higher thresholds of $\mathcal{A}_1$.

\subsection{Comparing the networks' external effects}

For a final result, we compare the networks' effects on a common, reliable external "actuator" neuron.
We consider the three corresponding networks $\mathcal{A}_1$, $\mathcal{A}_2$, and $\mathcal{D}$, defined as in Section~\ref{sec: correspondence}.
Fix $F_V$, $F_E$, $S_V$, and $S_E$.

Now, introduce a new neuron $a$, which we think of as an external actuator.
Suppose that $threshold^a = s_V$.
Connect one non-input neuron $v$ of $\mathcal{A}_1$ to the actuator neuron $a$, by a new edge with weight $1$.
Call the resulting network $\mathcal{A}^a_1$.
Likewise, connect the corresponding neuron $v$ of $\mathcal{A}_2$ is to neuron $a$, by a new edge with weight $1$, and call the resulting network $\mathcal{A}^a_2$.
For $\mathcal{D}$, connect all of the copies of $v$ to the same neuron $a$, by new edges of weight $1/m$, and call the resulting network $\mathcal{D}^a$.

Then neuron $a$ is a non-input neuron of each of the three networks $\mathcal{A}^a_1$,
$\mathcal{A}^a_2$, and $\mathcal{D}^a$.
We assume that, in all three networks, $a$ is reliable, i.e., it does not fail.
Also, the new edges from $v$ to $a$ and from copies of $v$ to $a$ do not fail.
%
Also, in all three networks, $firing^a(0) = 0$, i.e., $a$ does not fire at time $0$.

Fix any inputs of $\mathcal{A}_1$, the same inputs for $\mathcal{A}_2$, and corresponding inputs of $\mathcal{D}$.  We obtain:

\begin{theorem}
\label{thm: actuator}
Consider corresponding executions of the three networks $\mathcal{A}^a_1$, $\mathcal{A}^a_2$, and $\mathcal{D}^a$.
Then for any time $t \geq 1$,
\begin{enumerate}
\item 
If neuron $a$ fires at time $t$ in $\mathcal{A}^a_1$, then $a$ fires at time $t$ in $\mathcal{D}^a$.
\item 
If neuron $a$ does not fire at time $t$ in $\mathcal{A}^a_2$, then $a$ does not fire at time $t$ in $\mathcal{D}^a$.
\end{enumerate}
\end{theorem}

\begin{proof}
Fix time $t \geq 1$.

For Part 1, suppose that neuron $a$ fires at time $t$ in $\mathcal{A}^a_1$.
Then the firing threshold $s_V > 0$ of $a$ is met for time $t$, so it must be that neuron $v$ fires at time $t-1$ in $\mathcal{A}_1$.
Then Theorem~\ref{thm: firing} implies that at least $s_V m$ of the neurons in $copies(v)$ fire at time $t-1$ in $\mathcal{D}$.
Since the edges between these copies and $a$ in $\mathcal{D}^a$ all have weight $1/m$, and these edges do not fail, this means that the potential incoming to $a$ for time $t$ in $\mathcal{D}^a$ is at least $s_V$.
This is enough to meet $a$'s threshold of $s_V$, so $a$ fires at time $t$ in $\mathcal{D}^a$, as needed.

For Part 2, suppose that neuron $a$ does not fire at time $t$ in $\mathcal{A}^a_2$.
Then the firing threshold $s_V \leq 1$ of $a$ is not met for time $t$, so it must be that neuron $v$ does not fire at time $t-1$ in $\mathcal{A}_2$.
Then Theorem~\ref{thm: non-firing} implies that none of the neurons in $copies(v)$ fire
at time $t-1$ in $\mathcal{D}$.
This means that the potential incoming to $a$ for time $t$ in $\mathcal{D}^a$ is $0$.
This does not meet $a$'s threshold of $s_V > 0$, so $a$ does not fire at time $t$ in $\mathcal{D}^a$, as needed.
\end{proof}

\begin{example}
\emph{Line:}
\label{ex: line5}
Let $\mathcal{A}_1$ be the line network from Example~\ref{ex: line1}, with arbitrary $\lmax = 5$.
Let $\mathcal{A}_2$ be the corresponding lower-threshold version, and let $\mathcal{D}$ be the corresponding detailed version.
We assume arbitrary $m$, $s_V$, and $s_E$.
As observed before, for this example, the network $\mathcal{A}_2$ behaves identically to $\mathcal{A}_1$ on the same input pattern.

Now add the external actuator neuron $a$ to all three networks, by connecting neuron $\lmax$ in $\mathcal{A}_1$ and $\mathcal{A}_2$, and $copies(\lmax)$ in $\mathcal{D}$, to $a$.  
This yields $\mathcal{A}^a_1$, $\mathcal{A}^a_2$ and $\mathcal{D}^a$.

For the input pattern, let neuron $0$ fire at time $0$ in $\mathcal{A}^a_1$, and correspondingly for $\mathcal{A}^a_2$ and $\mathcal{D}$.
It is easy to see that, in the abstract networks $\mathcal{A}^a_1$ and $\mathcal{A}^a_2$, the actuator neuron $a$ fires at time $\lmax+1$ and at no other time.
Then Theorem~\ref{thm: actuator} implies that in $\mathcal{D}^a$, $a$ fires at time $\lmax+1$ and at no other time.
\end{example}

\section{Conclusions}
\label{sec: conclusions}

In this paper, we have illustrated how one might view brain algorithms using two levels of abstraction, where the higher, more abstract level is a simple, reliable model and the lower, more detailed level admits failures and compensates for them using redundancy.  We gave three theorems showing precise relationships between executions of the abstract and detailed models.
We presented this in the context of a particular discrete Spiking Neural Network (SNN) model, and a particular failure model (initial, permanent failures only), but we think that the general ideas should carry over to other network models and failure models. However, the technical notions of mapping for different models may be quite different, and this work remains to be done.

\paragraph{Future work:}
For SNNs, the first extension would be to accommodate more elaborate kinds of failures, such as allowing a bounded number of transient node and edge failures at each time $t$.  
We could also consider other anomalies such as spurious firing.

It remains to develop examples of how these correspondences can be used in interesting applications of SNNs; this paper contains only a few toy examples.  Applications to networks that solve interesting problems such as problems of decision-making or recognition would be desirable.

Also, as we noted in the Introduction, work by Valiant~\cite{Valiant} and work on the assembly calculus~\cite{PapadimitriouVempala, PVMM20} is carried out using multi-neuron representations.
It would be interesting to describe and analyze their algorithms using simpler, higher-level SNN models as an intermediate step.

Other types of neural network models should also be considered using levels of abstraction.
In particular, it would be interesting to study neural network models in which each neuron has a real-valued $rate$ state component that signifies its level of activity.
Such a component might change either continuously or in discrete steps.
Other real-valued components might also be included.
Such models are popular in some theoretical neuroscience research, such as~\cite{64cb611635384379b29c3433ed8f9ea6}.
In order to treat these models using levels of abstraction, we would have to introduce failures and redundancy into the detailed model.
This could end up being interesting, or tricky; for example, there is the danger that too much uncertainty in the detailed model might cause the rates in the detailed model to diverge unboundedly from those in the abstract model.
At any rate, this remains to be done.

\bibliography{Multi}

\end{document}